\documentclass[twoside]{article}

\usepackage{enumerate}
\usepackage[inline]{enumitem}
\usepackage{url}
\usepackage{natbib}
\usepackage{bbm}

\usepackage{graphicx}
\usepackage{tikz}
\usetikzlibrary{arrows,decorations.pathmorphing,backgrounds,positioning,fit,petri}
\usepackage{subcaption}
\graphicspath{{./figures/}}
\usepackage{array}

\usepackage{amsmath}
\usepackage{amsthm}
\usepackage{dsfont}
\theoremstyle{plain}

\newtheorem{proposition}{Proposition}

\usepackage{algorithm,algorithmic}

\usepackage[acronym,smallcaps,nowarn]{glossaries}

\usepackage{booktabs}

\DeclareRobustCommand{\parhead}[1]{\textbf{#1}~}

\usepackage{naesseth}
\usepackage{abbreviations}

\newcommand{\eps}{\ensuremath{\varepsilon}}
\newcommand{\Lo}{\ensuremath{\mathcal{L}}}    

\newcommand{\g}{\, ;}

\newacronym{ADVI}{advi}{automatic differentiation variational inference}

\newacronym{BBVI}{bbvi}{black-box variational inference}

\newacronym{CTM}{ctm}{correlated topic model}

\newacronym[\glslongpluralkey={deep exponential families}]{DEF}{def}{deep exponential family}
\newacronym{DMIS}{dmis}{deterministic multiple importance sampling}

\newacronym{ELBO}{elbo}{evidence lower bound}

\newacronym{GNTS}{gn-ts}{gamma-normal time series model}
\newacronym{G-REP}{g-rep}{generalized reparameterization}

\newacronym{KL}{kl}{{K}ullback-{L}eibler}

\newacronym{LDA}{lda}{latent {D}irichlet allocation}

\newacronym{MF}{mf}{matrix factorization}
\newacronym{MIS}{mis}{multiple importance sampling}
\newacronym{MATLAB}{matlab}{MATLAB}

\newacronym{NIPS}{nips}{Neural Information Processing Systems}

\newacronym{OBBVI}{o-bbvi}{overdispersed black-box variational inference}

\newacronym{RS-VI}{rsvi}{rejection sampling variational inference}

\newacronym{SVI}{svi}{stochastic variational inference}

\newacronym{VI}{vi}{variational inference}

\usepackage[accepted]{aistats2017}

\begin{document}

\twocolumn[

\aistatstitle{Reparameterization Gradients through Acceptance-Rejection Sampling Algorithms}

\aistatsauthor{Christian A. Naesseth$^{\dagger\ddagger}$ \And Francisco J. R. Ruiz$^{\ddagger\S}$ \And Scott W. Linderman$^{\ddagger}$  \And David M. Blei$^{\ddagger}$}

\aistatsaddress{$^\dagger$Link\"oping University~~$^\ddagger$Columbia University~~$^\S$University of Cambridge} ]

\begin{abstract}
Variational inference using the reparameterization trick has enabled large-scale approximate Bayesian inference in complex probabilistic models, leveraging stochastic optimization to sidestep intractable expectations. 
The reparameterization trick is applicable when we can simulate a random variable by applying a differentiable deterministic function on an auxiliary random variable whose distribution is fixed. For many distributions of interest (such as the gamma or Dirichlet), simulation of random variables relies on acceptance-rejection sampling. The discontinuity introduced by the accept--reject step means that standard reparameterization tricks are not applicable. We propose a new method that lets us leverage reparameterization gradients even when variables are outputs of a acceptance-rejection sampling algorithm. Our approach enables reparameterization on a larger class of variational distributions. In several studies of real and synthetic data, we show that the variance of the estimator of the gradient is significantly lower than other state-of-the-art methods. This leads to faster convergence of stochastic gradient variational inference.
\end{abstract}


\section{Introduction}\label{sec:introduction}

Variational inference~\citep{hinton1993,Waterhouse96,Jordan1999}
underlies many recent advances in large scale probabilistic modeling.
It has enabled sophisticated modeling of complex domains such as
images~\citep{Kingma2014} and text~\citep{Hoffman2013}. By definition,
the success of variational approaches depends on our ability to
\begin{enumerate*}[label=(\roman*)]
\item formulate a flexible parametric family of distributions; and
\item optimize the parameters to find the member of this family that
  most closely approximates the true posterior.
\end{enumerate*}
These two criteria are at odds---the more flexible the family, the
more challenging the optimization problem.  In this paper, we present
a novel method that enables more efficient optimization for a large
class of variational distributions, namely, for distributions that we
can efficiently simulate by acceptance-rejection sampling, or rejection sampling for short.

For complex models, the variational parameters can be optimized
by stochastic gradient ascent on the \gls{ELBO}, a lower bound on the 
marginal likelihood of the data.
There are two primary means of estimating the gradient of the \gls{ELBO}:
the score function estimator \citep{Paisley2012,Ranganath2014,Mnih2014} and the
reparameterization trick \citep{Kingma2014,Rezende2014,Price1958, Bonnet1964}, both of which rely on Monte Carlo sampling.  While the
reparameterization trick often yields lower variance estimates and
therefore leads to more efficient optimization, this approach has been
limited in scope to a few variational families (typically Gaussians).
Indeed, some lines of research have already tried to address this 
limitation \citep{Knowles2015,RuizTB2016}.

There are two requirements to apply the reparameterization trick.  The
first is that the random variable can be obtained through a
transformation of a simple random variable, such as a uniform or
standard normal; the second is that the transformation be
differentiable.  In this paper, we observe that all random variables
we simulate on our computers are ultimately transformations of
uniforms, often followed by accept-reject steps.  So if the
transformations are differentiable then we can use these existing simulation
algorithms to expand the scope of the reparameterization trick.

Thus, we show how to use existing rejection samplers to develop
stochastic gradients of variational parameters.  In short, each
rejection sampler uses a highly-tuned transformation that is
well-suited for its distribution.  We can construct new
reparameterization gradients by ``removing the lid'' from these black
boxes, applying $65+$ years of research on
transformations~\citep{vonneumann:51, devroye1986} to variational
inference. We demonstrate that this broadens the scope of variational
models amenable to efficient inference and provides lower-variance
estimates of the gradient compared to state-of-the-art approaches.

We first review variational inference, with a focus on stochastic
gradient methods. We then present our key contribution, \gls{RS-VI},
showing how to use efficient rejection samplers to produce
low-variance stochastic gradients of the variational objective.  We
study two concrete examples, analyzing rejection samplers for the
gamma and Dirichlet to produce new reparameterization gradients for
their corresponding variational factors. Finally, we analyze two
datasets with a \gls{DEF}~\citep{Ranganath2015}, comparing \gls{RS-VI}
to the state of the art.  We found that \gls{RS-VI} achieves a
significant reduction in variance and faster convergence of the
\gls{ELBO}. 
Code for all experiments is provided at
\url{github.com/blei-lab/ars-reparameterization}.

\section{Variational Inference}\label{sec:background}
Let $p(x,z)$ be a probabilistic model, \ie, a joint probability distribution of \emph{data} $x$ and \emph{latent} (unobserved) variables $z$. In Bayesian inference, we are interested in the posterior distribution $p(z|x) = \frac{p(x,z)}{p(x)}$. For most models, the posterior distribution is analytically intractable and we have to use an approximation, such as Monte Carlo methods or variational inference. In this paper, we focus on variational inference.

In variational inference, we approximate the posterior with a \emph{variational family} of distributions $q(z\g\theta)$, parameterized by $\theta$. Typically, we choose the \emph{variational parameters} $\theta$ that minimize the \gls{KL} divergence between $q(z\g\theta)$ and $p(z | x)$. This minimization is equivalent to maximizing the \gls{ELBO} \citep{Jordan1999}, defined as
\begin{align}
\begin{split}
\Lo(\theta) &= \E_{q(z\g \theta)}\left[f(z)\right] + \Ent[q(z\g \theta)],\\
f(z) & \eqdef \log p(x,z), \\
\quad\Ent[q(z\g \theta)] & \eqdef \E_{q(z\g \theta)}[- \log q(z\g \theta)].
\end{split}\label{eq:elbo}
\end{align}
When the model and variational family satisfy conjugacy requirements, we can use coordinate ascent to find a local optimum of the \gls{ELBO} \citep{BleiAlp2016}. If the conjugacy requirements are not satisfied, a common approach is to build a Monte Carlo estimator of the gradient of the \gls{ELBO} \citep{Paisley2012,Ranganath2014,Mnih2014,Salimans2013,Kingma2014}. This results in a stochastic optimization procedure, where different Monte Carlo estimators of the gradient amount to different algorithms. We review below two common estimators: the score function estimator and the reparameterization trick.%
\footnote{In this paper, we assume for simplicity that the gradient of the entropy $\grad_\theta \Ent[q(z\g \theta)]$ is available analytically. The method that we propose in Section~\ref{sec:method} can be easily extended to handle non-analytical entropy terms. Indeed, the resulting estimator of the gradient may have lower variance when the analytic gradient of the entropy is replaced by its Monte Carlo estimate. Here we do not explore that.}

\parhead{Score function estimator.}
The score function estimator, also known as the log-derivative trick or \textsc{reinforce} \citep{Williams1992,Glynn1990}, is a general way to estimate the gradient of the \gls{ELBO} \citep{Paisley2012,Ranganath2014,Mnih2014}. The score function estimator expresses the gradient as an expectation with respect to $q(z\g\theta)$:
\begin{align*}
\grad_\theta \Lo(\theta) = \E_{q(z\g\theta)}[ f(z) \grad_\theta \log q(z\g\theta)] + \grad_\theta \Ent[q(z\g \theta)].
\end{align*}
We then form Monte Carlo estimates by approximating the expectation with independent samples from the variational distribution. Though it is very general, the score function estimator typically suffers from high variance. In practice we also need to apply variance reduction techniques such as Rao-Blackwellization \citep{Casella1996} and control variates \citep{robert2004monte}.

\parhead{Reparameterization trick.}
The reparameterization trick \citep{Salimans2013,Kingma2014,Price1958,Bonnet1964} results in a lower variance estimator compared to the score function, but it is not as generally applicable. It requires that: (i) the latent variables $z$ are continuous; and (ii) we can simulate from $q(z\g\theta)$ as follows, 
\begin{align}
z = h(\eps,\theta),\qquad \textrm{with } \eps \sim s(\eps).
\label{eq:mapping}
\end{align}
Here, $s(\eps)$ is a distribution that does not depend on the variational parameters; it is typically a standard normal or a standard uniform. Further, $h(\eps,\theta)$ must be differentiable with respect to $\theta$. In statistics, this is known as a non-central parameterization and has been shown to be helpful in, \eg, Markov chain Monte Carlo methods \citep{bernardo2003non}.

Using \eqref{eq:mapping}, we can move the derivative inside the expectation and rewrite the gradient of the \gls{ELBO} as
\begin{align*}
&\grad_\theta \Lo(\theta) = \E_{ s(\eps) }\left[\grad_z f(h(\eps,\theta)) \grad_\theta h(\eps,\theta)\right] + \grad_\theta \Ent[q(z\g \theta)].
\end{align*}
Empirically, the reparameterization trick has been shown to be beneficial over direct Monte Carlo estimation of the gradient using the score fuction estimator \citep{Salimans2013,Kingma2014,Titsias2014_doubly,Fan2015}. Unfortunately, many distributions commonly used in variational inference, such as gamma or Dirichlet, are not amenable to standard reparameterization because samples are generated using a rejection sampler \citep{vonneumann:51,robert2004monte}, introducing discontinuities to the mapping. We next show that taking a novel view of the acceptance-rejection sampler lets us perform exact reparameterization.


\section{Reparameterizing the Acceptance-Rejection Sampler}\label{sec:method}


\begin{figure*}[t]
  \centering
  \includegraphics[width=6.5in]{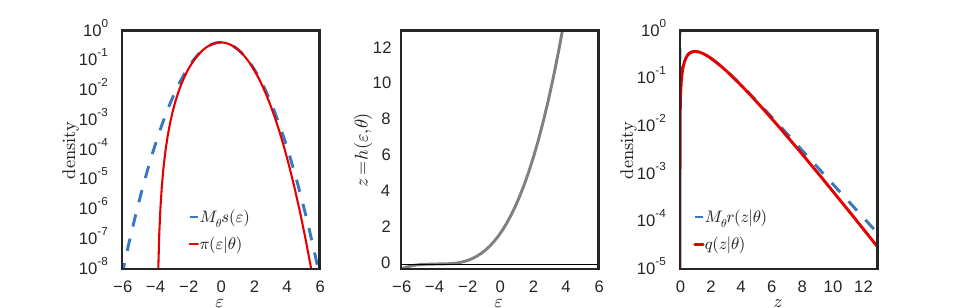} 
  \caption{Example of a reparameterized rejection sampler for~${q(z \g
      \theta)= \gam(\theta,1)}$, shown here with~${\theta = 2}$.  We
    use the rejection sampling algorithm of \citet{Marsaglia:2000},
    which is based on a nonlinear transformation ${h(\varepsilon,
      \theta)}$ of a standard normal~${\varepsilon \sim \N(0,1)}$
    (c.f. Eq.~\ref{eq:marsaglia}), and has acceptance probability
    of~$0.98$ for~${\theta=2}$. The marginal density of the accepted
    value of~$\varepsilon$ (integrating out the acceptance
    variables,~$u_{1:i}$) is given by~${\pi(\varepsilon \g
      \theta)}$. We compute unbiased estimates of the gradient of the
    \gls{ELBO}~\eqref{eq:reparamgrad} via Monte Carlo, using
    Algorithm~\ref{alg:rejectionsampling} to rejection
    sample~${\varepsilon \sim \pi(\varepsilon \g \theta)}$. By
    reparameterizing in terms of~$\varepsilon$, we obtain a
    low-variance estimator of the gradient for challenging
    variational distributions.}
  \label{fig:real_gamma}
\end{figure*}

The basic idea behind reparameterization is to rewrite simulation from a complex distribution as a deterministic mapping of its parameters and a set of simpler random variables. We can view the rejection sampler as a complicated deterministic mapping of a (random) number of simple random variables such as uniforms and normals. This makes it tempting to take the standard reparameterization approach when we consider random variables generated by rejection samplers. However, this mapping is in general not continuous, and thus moving the derivative inside the expectation and using direct automatic differentiation would not necessarily give the correct answer.

Our insight is that we can overcome this problem by instead considering only the marginal over the accepted sample, analytically integrating out the accept-reject variable. Thus, the mapping comes from the proposal step. This is continuous under mild assumptions, enabling us to greatly extend the class of variational families amenable to reparameterization.

We first review rejection sampling and present the reparameterized rejection sampler. Next we show how to use it to calculate low-variance gradients of the \gls{ELBO}. Finally, we present the complete stochastic optimization for variational inference, \gls{RS-VI}.

\subsection{Reparameterized Rejection Sampling}
Acceptance-Rejection sampling is a powerful way of simulating random variables from complex distributions whose inverse cumulative distribution functions are not available or are too expensive to evaluate \citep{devroye1986,robert2004monte}. We consider an alternative view of rejection sampling in which we explicitly make use of the reparameterization trick. This view of the rejection sampler enables our variational inference algorithm in Section~\ref{subsec:method}. 

To generate samples from a distribution $q(z\g\theta)$ using rejection sampling, we first sample from a \emph{proposal distribution} $r(z\g \theta)$ such that $q(z\g \theta) \leq M_\theta r(z\g \theta)$ for some $M_\theta <\infty$. In our version of the rejection sampler, we assume that the proposal distribution is reparameterizable, \ie, that generating $z\sim r(z\g \theta)$ is equivalent to generating $\eps \sim s(\eps)$ (where $s(\eps)$ does not depend on $\theta$) and then setting $z=h(\eps,\theta)$ for a differentiable function $h(\eps,\theta)$. We then accept the sample with probability $\min\left\{1, \frac{q\left(h(\eps,\theta)\g \theta\right)}{M_\theta r\left(h(\eps,\theta) \g \theta\right)}\right\}$; otherwise, we reject the sample and repeat the process. We illustrate this in Figure~\ref{fig:real_gamma} and provide a summary of the method in Algorithm~\ref{alg:rejectionsampling}, where we consider the output to be the (accepted) variable $\eps$, instead of $z$.


The ability to simulate from $r(z\g \theta)$ by a reparameterization through a differentiable $h(\eps,\theta)$ is not needed for the rejection sampler to be valid. However, this is indeed the case for the rejection sampler of many common distributions.

\begin{algorithm}[t]
\caption{Reparameterized Rejection Sampling}\label{alg:rejectionsampling}
\begin{algorithmic}[1]
\REQUIRE target $q(z\g \theta)$, proposal $r(z\g \theta)$, and constant $M_\theta$, with $q(z\g \theta) \leq M_\theta r(z\g \theta)$ 
\ENSURE $\eps$ such that $h(\eps,\theta) \sim q(z\g \theta)$
\STATE $i \gets 0$
\REPEAT 
\STATE $i \gets i +1 $
\STATE Propose $\eps_i \sim s(\eps)$
\STATE Simulate $u_i \sim \uni[0,1]$
\UNTIL $u_i < \frac{q\left(h(\eps_i,\theta)\g \theta\right)}{M_\theta r\left(h(\eps_i,\theta) \g \theta\right)}$
\RETURN $\eps_i$
\end{algorithmic}
\end{algorithm}

%
\subsection{The Reparameterized Rejection Sampler in Variational Inference}
\label{subsec:method}

We now use reparameterized rejection sampling to develop a novel Monte Carlo estimator of the gradient of the \gls{ELBO}. We first rewrite the \gls{ELBO} in \eqref{eq:elbo} as an expectation in terms of the transformed variable $\eps$,
\begin{align}
\begin{split}
&\Lo(\theta) = \E_{q(z\g \theta)}\left[f(z)\right] + \Ent[q(z\g \theta)] \\
&\quad\quad=\E_{\pi(\eps\g\theta)}\left[f\left(h(\eps,\theta)\right)\right] + \Ent[q(z\g \theta)].
\end{split}
\label{eq:ReparamExp}
\end{align}
In this expectation, $\pi(\eps\g\theta)$ is the distribution of the \emph{accepted sample} $\eps$ in Algorithm~\ref{alg:rejectionsampling}. We construct it by marginalizing over the auxiliary uniform variable $u$,
\begin{align}
\pi(\eps\g\theta) &= \int \pi(\eps,u\g\theta) \myd u \nonumber\\
&= \int M_\theta s(\eps) 
\mathds{1}
\left[0 < u < \frac{q\left(h(\eps,\theta)\g \theta\right)}{M_\theta r\left(h(\eps,\theta)\g \theta\right)}\right] \myd u \nonumber\\
&= s(\eps) \frac{q\left(h(\eps,\theta)\g \theta\right)}{r\left(h(\eps,\theta)\g \theta\right)},
\label{eq:distrejection}
\end{align}
where $\mathds{1}[x \in A]$ is the indicator function, and recall that $M_\theta$ is a constant used in the rejection sampler. 
This can be seen by the algorithmic definition of the rejection sampler, where we propose values $\eps \sim s(\eps)$ and $u \sim \uni[0,1]$ until acceptance, \ie, until $u < \frac{q\left(h(\eps,\theta)\g \theta\right)}{M_\theta r\left(h(\eps,\theta) \g \theta\right)}$. Eq.~\ref{eq:ReparamExp} follows intuitively, but we formalize it in Proposition~\ref{thm:reparam}.



\begin{proposition}
\label{thm:reparam}
Let $f$ be any measurable function, and $\eps \sim \pi(\eps\g\theta)$, defined by \eqref{eq:distrejection} (and implicitly by Algorithm~\ref{alg:rejectionsampling}). Then
\begin{align*}
&\E_{\pi(\eps\g\theta)}\left[f\left(h(\eps,\theta)\right)\right] = \int f(z) q(z\g \theta) \myd z.
\end{align*}
\end{proposition}
\begin{proof}
Using the definition of $\pi(\eps\g\theta)$,
\begin{align*}
&\E_{\pi(\eps\g\theta)}\left[f\left(h(\eps,\theta)\right)\right] 
=\int f\left(h(\eps,\theta)\right) s(\eps) \frac{q\left(h(\eps,\theta)\g \theta\right)}{ r\left(h(\eps,\theta)\g \theta\right)}  \myd \eps \\
&= \int  f(z) r(z\g \theta) \frac{q(z\g \theta)}{r(z\g \theta)} \myd z = \int f(z) q(z\g \theta) \myd z,
\end{align*}
where the second to last equality follows because ${h(\eps,\theta), \eps \sim s(\eps)}$ is a reparameterization of $r(z\g \theta)$.
\end{proof}
We can now compute the gradient of $\E_{q(z\g \theta)}[f(z)]$ based on Eq.~\ref{eq:ReparamExp},
\begin{align}
\begin{split}
&\grad_\theta \E_{q(z\g \theta)}[f(z)] = \grad_\theta \E_{\pi(\eps\g\theta)}[f(h(\eps,\theta))] \\
&= \underbrace{\E_{\pi(\eps\g\theta)}[\grad_\theta f(h(\eps,\theta))]}_{\defeq g_{\text{rep}}}+ \\
&\quad + \underbrace{\E_{\pi(\eps\g\theta)}\left[ f(h(\eps,\theta)) \grad_\theta  \log \frac{q(h(\eps,\theta)\g \theta)}{r(h(\eps,\theta)\g \theta)} \right]}_{\defeq g_{\text{cor}}},
\end{split}\label{eq:graddef}
\end{align}
where we have used the log-derivative trick and rewritten the integrals as expectations with respect to $\pi(\eps\g\theta)$ (see the supplement for all details.) We define $g_{\text{rep}}$ as the reparameterization term, which takes advantage of gradients with respect to the model and its latent variables; we define $g_{\text{cor}}$ as a correction term that accounts for \emph{not} using $r(z\g \theta) \equiv q(z\g \theta)$.



Using \eqref{eq:graddef}, the gradient of the \gls{ELBO} in \eqref{eq:elbo} can be written as
\begin{align}
\begin{split}
&\grad_\theta \Lo(\theta)  = g_{\text{rep}}+ g_{\text{cor}}+\grad_\theta \Ent[q(z\g\theta)],
\end{split}
\label{eq:reparamgrad}
\end{align}
and thus we can build an unbiased one-sample Monte Carlo estimator $\hat g \approx \grad_\theta \Lo(\theta)$ as
\begin{align}
\begin{split}
&\hat g  \eqdef \hat g_{\text{rep}} + \hat g_{\text{cor}}+\grad_\theta \Ent[q(z\g\theta)], \\
& \hat g_{\text{rep}} = \grad_z f(z)\big|_{z=h(\eps,\theta)} \grad_\theta h(\eps,\theta) \\
& \hat g_{\text{cor}} = f(h(\eps,\theta)) \grad_\theta  \log \frac{q(h(\eps,\theta)\g \theta)}{r(h(\eps,\theta)\g \theta)},
\end{split}
\label{eq:reparamgrad2}
\end{align}
where $\eps$ is a sample generated using Algorithm~\ref{alg:rejectionsampling}. Of course, one could generate more samples of $\eps$ and average, but we have found a single sample to suffice in practice.



Note if $h(\eps,\theta)$ is invertible in $\eps$ then we can simplify the evaluation of the gradient of the log-ratio in $g_{\text{cor}}$,
\begin{align}
&\grad_\theta  \log \frac{q(h(\eps,\theta)\g \theta)}{r(h(\eps,\theta)\g \theta)}  = \nonumber\\
&\quad\quad\quad\grad_\theta  \log q(h(\eps,\theta)\g \theta) + \grad_\theta \log \left| \frac{d h}{d\eps}(\eps,\theta) \right| .\label{eq:gradQRcorr}
\end{align}
See the supplementary material for details.

Alternatively, we could rewrite the gradient as an expectation with respect to $s(\eps)$ (this is an intermediate step in the derivation shown in the supplement),
\begin{align}
&\grad_\theta \E_{q(z\g \theta)}[f(z)] = \E_{s(\eps)}\left[ \frac{q\left(h(\eps,\theta)\g \theta\right)}{r\left(h(\eps,\theta)\g \theta\right)}  \grad_\theta f\left(h(\eps,\theta)\right)  \right]+ \nonumber\\
&+\E_{s(\eps)}\left[ \frac{q\left(h(\eps,\theta)\g \theta\right)}{r\left(h(\eps,\theta)\g \theta\right)}   f\left(h(\eps,\theta)\right) \grad_\theta \log  \frac{q\left(h(\eps,\theta)\g \theta\right)}{r\left(h(\eps,\theta)\g \theta\right)}  \right],\nonumber
\end{align}
and build an importance sampling-based Monte Carlo estimator, in which the importance weights would be $q\left(h(\eps,\theta)\g \theta\right) / r\left(h(\eps,\theta)\g \theta\right)$.
However, we would expect this approach to be beneficial for low-dimensional problems only, since for high-dimensional $z$ the variance of the importance weights would be too high.

\begin{algorithm}
\caption{Rejection Sampling Variational Inference}\label{alg:svi}
\begin{algorithmic}[1]
\REQUIRE Data $x$, model $p(x,z)$, variational family $q(z\g \theta)$
\ENSURE Variational parameters $\theta^*$
\REPEAT
\STATE Run Algorithm~\ref{alg:rejectionsampling} for $\theta^n$ to obtain a sample $\eps$
\STATE Estimate the gradient $\hat g^n$ at $\theta = \theta^n$ (Eq.~\ref{eq:reparamgrad2})
\STATE Calculate the stepsize $\rho^n$ (Eq.~\ref{eq:stepsize})
\STATE Update $\theta^{n+1} = \theta^n + \rho^n \hat{g}^n$
\UNTIL \textbf{convergence}
\end{algorithmic}
\end{algorithm}

\subsection{Full Algorithm}
We now describe the full variational algorithm based on reparameterizing the rejection sampler. In Section~\ref{sec:examples} we give concrete examples of how to reparameterize common variational families.

We make use of Eq.~\ref{eq:reparamgrad} to obtain a Monte Carlo estimator of the gradient of the \gls{ELBO}. We use this estimate to take stochastic gradient steps. We use the step-size sequence $\rho^n$ proposed by \citet{Kucukelbir2016} (also used by \citet{RuizTB2016}), which combines \textsc{rmsprop} \citep{Tieleman2012} and Adagrad \citep{Duchi2011}. It is
\begin{align}
\begin{split}
\rho^n &= \eta \cdot n^{-1/2 + \delta} \cdot \left(1 + \sqrt{s^n}\right)^{-1},\\
s^n &= t \left( \hat{g}^n \right)^2 + (1-t) s^{n-1},
\label{eq:stepsize}
\end{split}
\end{align}
where $n$ is the iteration number. We set $\delta = 10^{-16}$ and $t=0.1$, and we try different values for $\eta$. (When~$\theta$ is a vector, the operations above are element-wise.)

We summarize the full method in Algorithm~\ref{alg:svi}. We refer to our method as \gls{RS-VI}. 

\section{Related Work}\label{sec:relatedwork}
The reparameterization trick has also been used in \gls{ADVI} \citep{Kucukelbir2015,Kucukelbir2016}. \gls{ADVI} applies a transformation to the random variables such that their support is on the reals and then places a Gaussian variational posterior approximation over the transformed variable~$\eps$. In this way, \gls{ADVI} allows for standard reparameterization, but it cannot fit gamma or Dirichlet variational posteriors, for example. Thus, \gls{ADVI} struggles to approximate probability densities with singularities, as noted by \citet{RuizTB2016}. In contrast, our approach allows us to apply the reparameterization trick on a wider class of variational distributions, which may be more appropriate when the exact posterior exhibits sparsity.

In the literature, we can find other lines of research that focus on extending the reparameterization gradient to other distributions. For the gamma distribution, \citet{Knowles2015} proposed a method based on approximations of the inverse cumulative density function; however, this approach is limited only to the gamma distribution and it involves expensive computations. 
For general expectations, \citet{Schulman2015} expressed the gradient as a sum of a reparameterization term and a correction term to automatically estimate the gradient 
 in the context of stochastic computation graphs. However, it is not possible to directly apply it to variational inference with acceptance-rejection sampling. This is due to discontinuities in the accept--reject step and the fact that a rejection sampler produces a \emph{random number} of random variables. 
Recently, another line of work has focused on applying reparameterization to discrete latent variable models \citep{Maddison2017,Jang2017} through a continuous relaxation of the discrete space.

The \gls{G-REP} method \citep{RuizTB2016} exploits the decomposition of the gradient as $g_{\text{rep}}+g_{\text{cor}}$ by applying a transformation based on standardization of the sufficient statistics of $z$. Our approach differs from \gls{G-REP}: instead of searching for a transformation of $z$ that makes the distribution of $\eps$ weakly dependent on the variational parameters (namely, standardization), we do the opposite by choosing a transformation of a simple random variable $\eps$ such that the distribution of $z=h(\eps,\theta)$ is \emph{almost} equal to $q(z\g\theta)$. For that, we reuse the transformations typically used in rejection sampling. Rather than having to derive a new transformation for each variational distribution, we leverage decades of research on transformations in the rejection sampling literature \citep{devroye1986}. In rejection sampling, these transformations (and the distributions of $\eps$) are chosen so that they have high acceptance probability, which means we should expect to obtain $g_{\text{cor}}\approx 0$ with \gls{RS-VI}. In Sections~\ref{sec:examples} and \ref{sec:experiments} we compare \gls{RS-VI} with \gls{G-REP} and show that it exhibits significantly lower variance, thus leading to faster convergence of the inference algorithm.

Finally, another line of research in non-conjugate variational inference aims at developing more expressive variational families \citep{Salimans2015,Tran2016,Maaloe2016,Ranganath2016}. \gls{RS-VI} can extend the reparameterization trick to these methods as well, whenever rejection sampling is used to generate the random variables.


\section{Examples of Acceptance-Rejection Reparameterization}
\label{sec:examples}

As two examples, we study rejection sampling and reparameterization of two well-known distributions: the gamma and Dirichlet. These have been widely used as variational families for approximate Bayesian inference. 
We emphasize that \gls{RS-VI} is not limited to these two cases, it applies to any variational family $q(z\g\theta)$ for which a reparameterizable rejection sampler exists. We provide other examples in the supplement.

\subsection{Gamma Distribution}
\label{subsec:gamma_example}

One of the most widely used rejection sampler is for the gamma distribution. Indeed, the gamma distribution is also used in practice to generate \eg beta, Dirichlet, and Student's t-distributed random variables. The gamma distribution, $\gam(\alpha,\beta)$, is defined by its shape $\alpha$ and rate $\beta$.

For $\gam(\alpha,1)$ with $\alpha \geq 1$, \citet{Marsaglia:2000} developed an efficient rejection sampler. It uses a truncated version of the following reparameterization
\begin{align}
\label{eq:marsaglia}
  z &= h_{\gam}(\eps,\alpha) \eqdef \left(\alpha-\frac{1}{3}\right)\left(1+\frac{\eps}{\sqrt{9\alpha-3}}\right)^3, \\
\nonumber &\eps \sim s(\eps) \eqdef \N(0,1).
\end{align}
When $\beta \neq 1$, we divide $z$ by the rate $\beta$ and obtain a sample distributed as $\gam(\alpha,\beta)$. The acceptance probability is very high: it exceeds $0.95$ and $0.98$ for $\alpha=1$ and $\alpha=2$, respectively. In fact, as $\alpha \to \infty$ we have that $\pi(\eps\g\theta) \to s(\eps)$, which means that the acceptance probability approaches $1$. Figure~\ref{fig:real_gamma} illustrates the involved functions and distributions for shape $\alpha=2$.

For $\alpha < 1$, we observe that $z = u^{1/\alpha} \tilde z$ is distributed as $\gam(\alpha,\beta)$ for $\tilde z \sim \gam(\alpha+1,\beta)$ and ${u \sim \uni[0,1]}$ \citep{stuart1962,devroye1986}, and apply the rejection sampler above for $\tilde z$.

\begin{figure}[t]
\includegraphics[width=0.9\columnwidth]{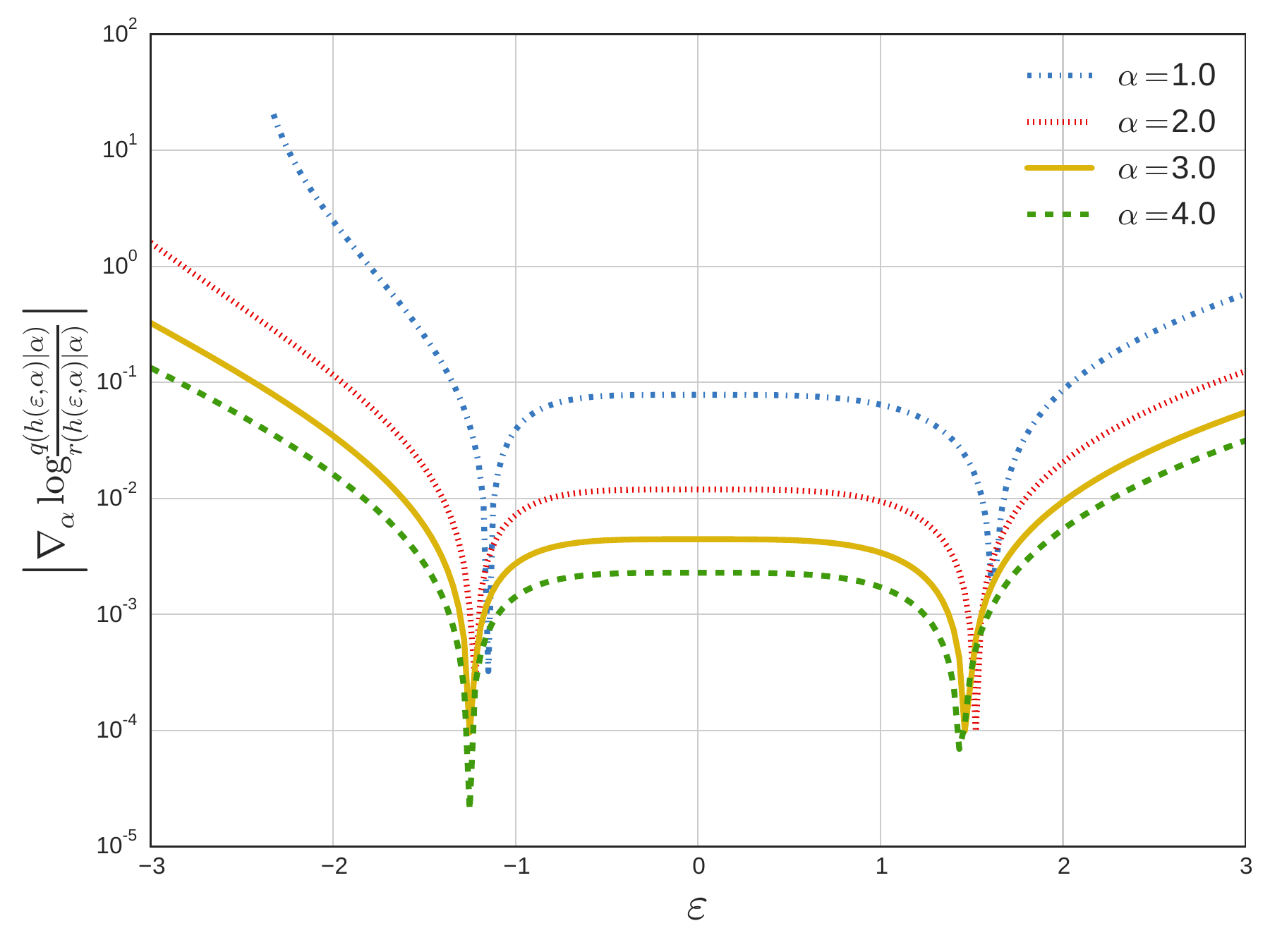}
\vspace*{-10pt}
\caption{The correction term of \gls{RS-VI}, and as a result the gradient variance, decreases  with increasing shape $\alpha$. We plot absolute value of the gradient of the log-ratio between the target (gamma) and proposal distributions as a function of $\eps$.
}\label{fig:gammacorr}
\end{figure}


We now study the quality of the transformation in \eqref{eq:marsaglia} for different values of the shape parameter $\alpha$. Since $\pi(\eps\g\theta) \to s(\eps)$ as $\alpha\to\infty$, we should expect the correction term $g_{\text{cor}}$ to decrease with $\alpha$. We show that in Figure~\ref{fig:gammacorr}, where we plot the log-ratio \eqref{eq:gradQRcorr} from the correction term as a function of $\eps$ for four values of~$\alpha$. We additionally show in Figure~\ref{fig:compare_qeps} that the distribution $\pi(\eps\g\theta)$ converges to $s(\eps)$ (a standard normal) as $\alpha$ increases. For large $\alpha$, $\pi(\eps\g\theta)\approx s(\eps)$ and the acceptance probability of the rejection sampler approaches~$1$, which makes the correction term negligible. In Figure~\ref{fig:compare_qeps}, we also show that $\pi(\eps\g\theta)$ converges faster to a standard normal than the standardization procedure used in \gls{G-REP}. We exploit this property---that performance improves with $\alpha$---to artificially increase the shape for any gamma distribution. We now explain this trick, which we call \emph{shape augmentation}.
\begin{figure*}[t]
  \centering
  \includegraphics[width=1.5\columnwidth]{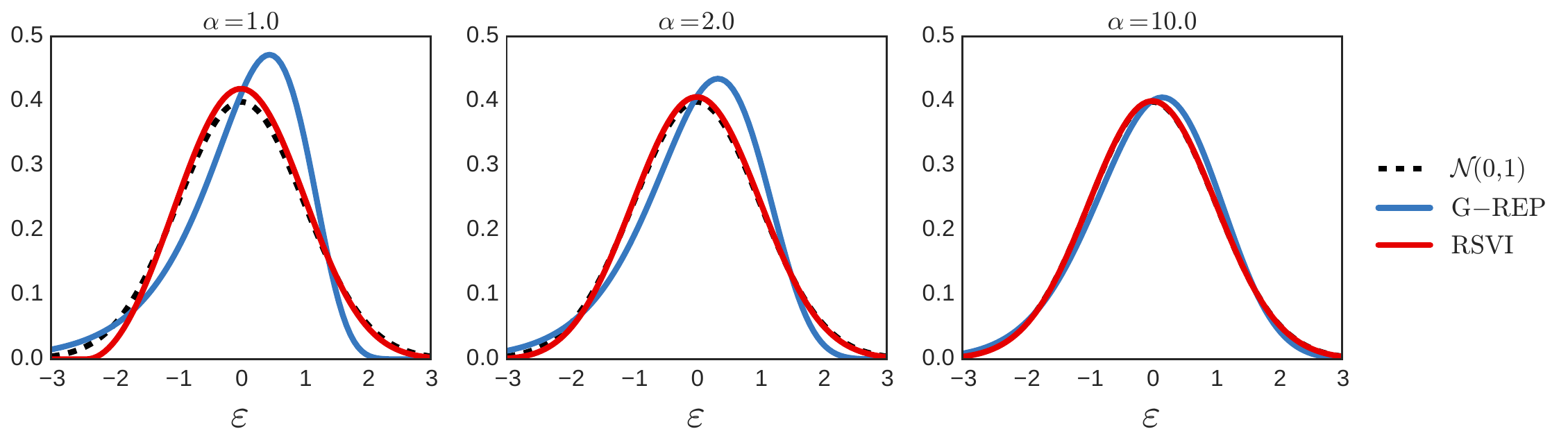} 
  \vspace*{-10pt}
  \caption{In the distribution on the transformed space $\eps$ for a gamma distribution we can see that the rejection sampling-inspired transformation converges faster to a standard normal. Therefore it is less dependent on the parameter $\alpha$, which implies a smaller correction term. We compare the transformation of \gls{RS-VI} (this paper) with the standardization procedure suggested in \gls{G-REP} \citep{RuizTB2016}, for shape parameters $\alpha = \{1,2,10\}$.}
  \label{fig:compare_qeps}
\end{figure*}


\parhead{Shape augmentation.}
Here we show how to exploit the fact that the rejection sampler improves for increasing shape $\alpha$. We make repeated use of the trick above, using uniform variables, to control the value of $\alpha$ that goes into the rejection sampler. That is, to compute the \gls{ELBO} for a $\gam(\alpha,1)$ distribution, we can first express the random variable as $z = \tilde z \prod_{i=1}^B u_i^{\frac{1}{\alpha+i-1}}$ (for some positive integer $B$), $\tilde z \sim \gam(\alpha+B,1)$ and $u_i \iidsim \uni[0,1]$. This can be proved by induction, since $\tilde z u_B^{\frac{1}{\alpha+B-1}} \sim \gam(\alpha+B-1,1)$, $\tilde z u_B^{\frac{1}{\alpha+B-1}} u_{B-1}^{\frac{1}{\alpha+B-2}} \sim \gam(\alpha+B-2,1)$, \etc. Hence, we can apply the rejection sampling framework for $\tilde z \sim \gam(\alpha+B,1)$ instead of the original $z$. We study the effect of shape augmentation on the variance in Section~\ref{subsec:dirichlet_example}.

\subsection{Dirichlet Distribution}
\label{subsec:dirichlet_example}

The $\diri(\alpha_{1:K})$ distribution, with concentration parameters $\alpha_{1:K}$, is a $K$-dimensional multivariate distribution with ${K-1}$ degrees of freedom. To simulate random variables we use the fact that if ${\tilde z_k \sim \gam(\alpha_k,1)}$ \iid, then ${z_{1:K} = \left(\sum_\ell \tilde z_\ell\right)^{-1}(\tilde z_1,\ldots,\tilde z_K)^\top \sim \diri(\alpha_{1:K})}$. 

Thus, we make a change of variables to reduce the problem to that of simulating independent gamma distributed random variables,
\begin{align*}
&\E_{q(z_{1:K}\g\alpha_{1:K})}[f(z_{1:K})] =\\
&= \int f\left( \frac{\tilde z_{1:K}}{\sum_{\ell=1}^K \tilde z_\ell}\right)\prod_{k=1}^K\gam(\tilde z_k\g\alpha_k,1) \myd \tilde z_{1:K}.
\end{align*}
We apply the transformation in Section~\ref{subsec:gamma_example} for the gamma-distributed variables, $\tilde z_k = h_{\gam}(\eps_k,\alpha_k)$, where the variables $\eps_k$ are generated by independent gamma rejection samplers.
\begin{figure}[t]
\includegraphics[width=0.9\columnwidth]{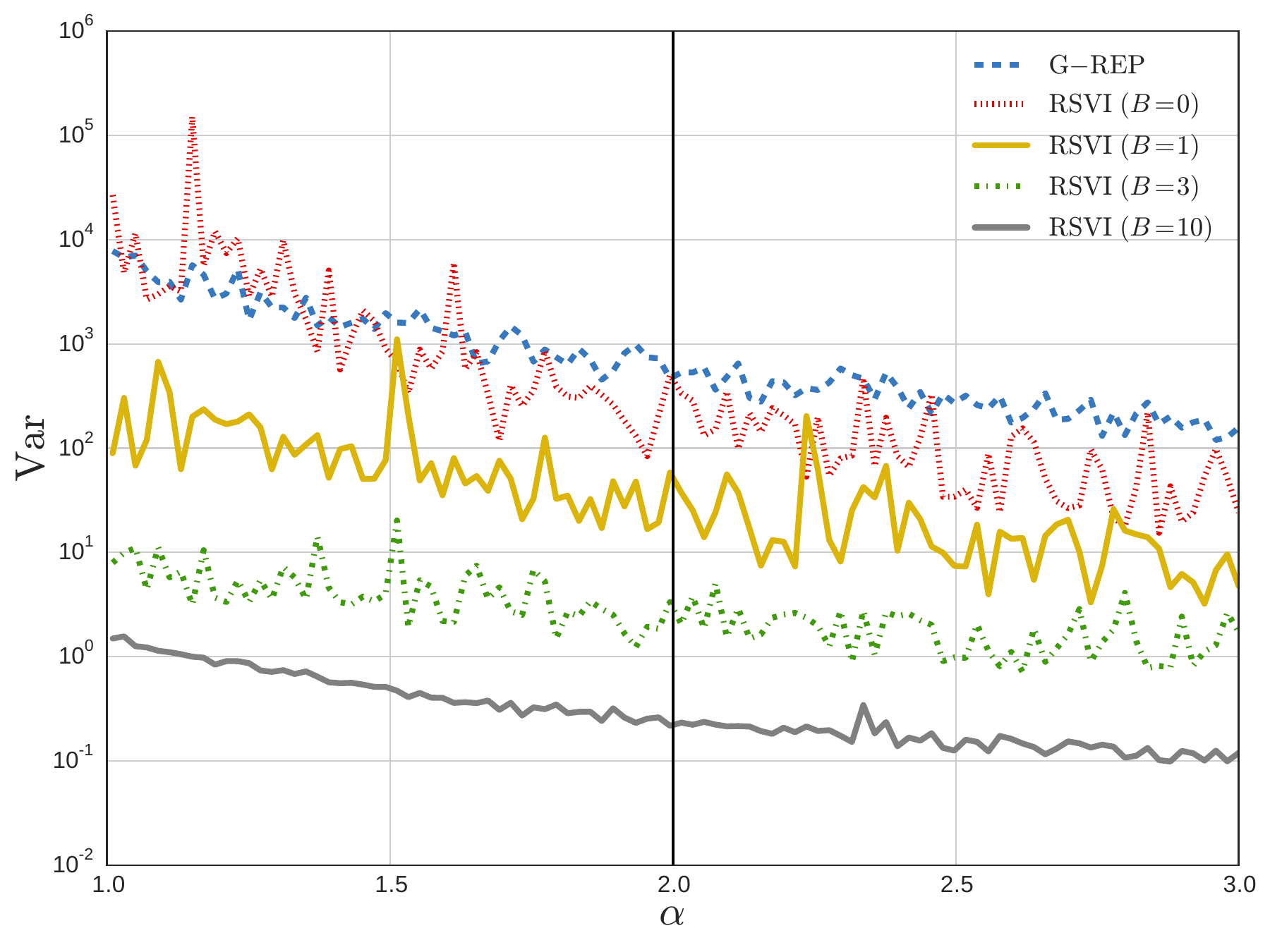}
\caption{\gls{RS-VI} (this paper) achieves lower variance compared to \gls{G-REP} \citep{RuizTB2016}. The estimated variance is for the first component of Dirichlet approximation to a multinomial likelihood with uniform Dirichlet prior.  Optimal concentration is $\alpha=2$, and $B$ denotes shape augmentation. 
}\label{fig:dirichlet}
\end{figure}
To showcase this, we study a simple conjugate model where the exact gradient and posterior are available: a multinomial likelihood with Dirichlet prior and Dirichlet variational distribution. In Figure~\ref{fig:dirichlet} we show the resulting variance of the first component of the gradient, based on simulated data from a Dirichlet distribution with $K=100$ components, uniform prior, and $N=100$ trials. We compare the variance of \gls{RS-VI} (for various shape augmentation settings) with the \gls{G-REP} approach \citep{RuizTB2016}. \gls{RS-VI} performs better even without the augmentation trick, and significantly better with it.

\section{Experiments}\label{sec:experiments}
\glsresetall

In Section~\ref{sec:examples} we compared \gls{RS-VI} with \gls{G-REP} and found a substantial variance reduction on synthetic examples. Here we evaluate \gls{RS-VI} on a more challenging model, the sparse gamma \gls{DEF}~\citep{Ranganath2015}. On two real datasets, we compare \gls{RS-VI} with state-of-the-art methods: \gls{ADVI} \citep{Kucukelbir2015,Kucukelbir2016}, \gls{BBVI} \citep{Ranganath2014}, and \gls{G-REP} \citep{RuizTB2016}.

\parhead{Data.}
The datasets we consider are the Olivetti faces\footnote{\url{http://www.cl.cam.ac.uk/research/dtg/attarchive/facedatabase.html}} and \gls{NIPS} 2011 conference papers. The Olivetti faces dataset consists of $64\times 64$ gray-scale images of human faces in $8$ bits, \ie, the data is discrete and in the set $\{0,\ldots,255\}$. In the \gls{NIPS} dataset we have documents in a bag-of-words format with an effective vocabulary of $5715$ words. 
\begin{figure}[t]
\centering
\includegraphics[width=0.9\columnwidth]{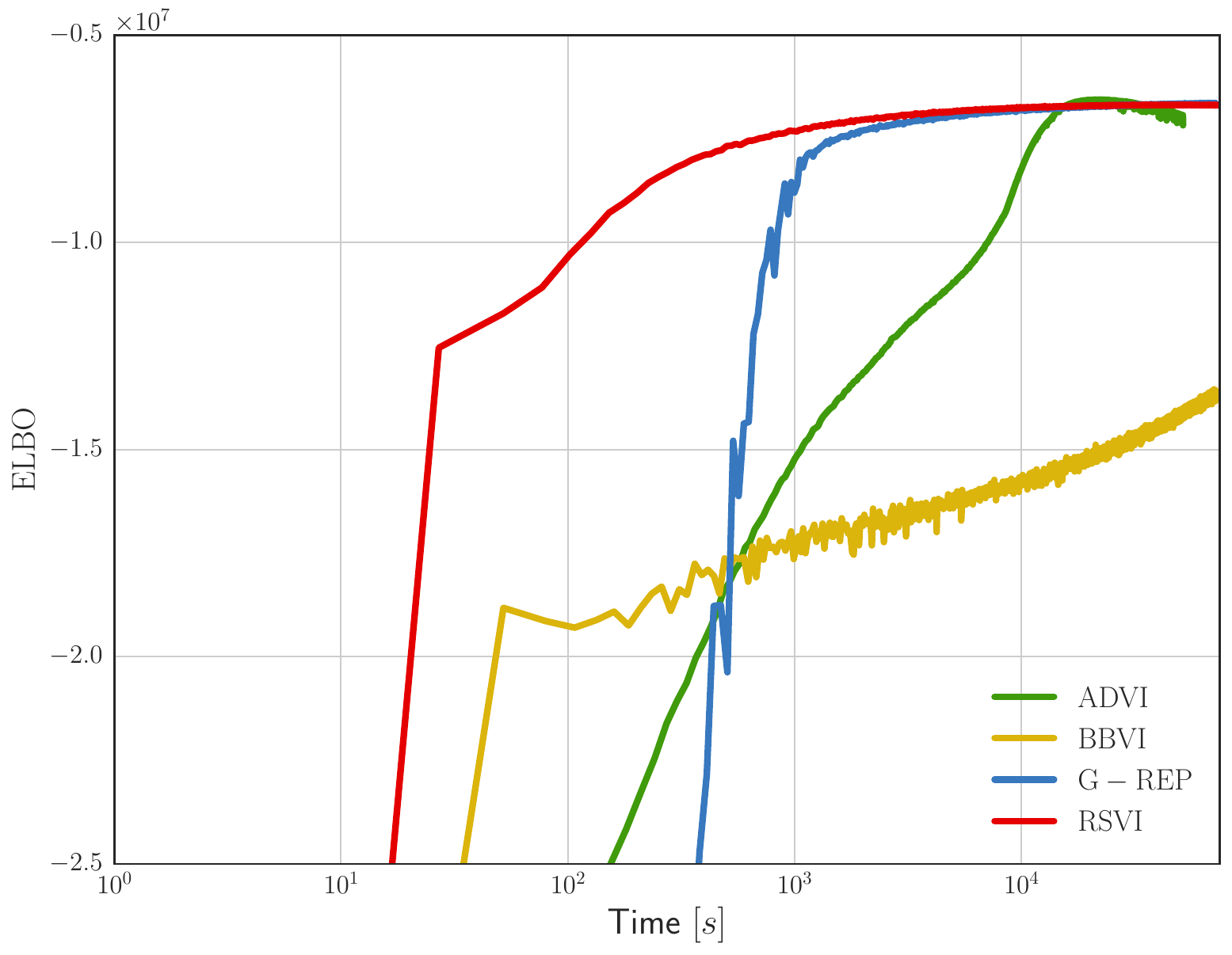}
\caption{\gls{RS-VI} (this paper) presents a significantly faster initial improvement of the \gls{ELBO} as a function of wall-clock time. The model is a sparse gamma \gls{DEF}, applied to the Olivetti faces dataset, and we compare with \gls{ADVI} \citep{Kucukelbir2016}, \gls{BBVI} \citep{Ranganath2014}, and \gls{G-REP} \citep{RuizTB2016}. }\label{fig:olivetteDEF}
\end{figure}

\begin{table*}[t]
\centering
\begin{subtable}{0.45\textwidth}
\begin{tabular}{cccc}
\toprule
 & \gls{RS-VI} $B=1$ & \gls{RS-VI} $B=4$ &\gls{G-REP}\\ \hline
Min & $6.0\mathrm{e}{-4}$ & $1.2\mathrm{e}{-3}$ & $2.7\mathrm{e}{-3}$\\
Median & $\mathbf{9.0\mathrm{e}{7}}$ & $\mathbf{2.9\mathrm{e}{7}}$ & $1.6\mathrm{e}{12}$\\
Max & $1.2\mathrm{e}{17}$ & $\mathbf{3.4\mathrm{e}{14}}$ & $1.5\mathrm{e}{17}$ \\ \bottomrule 
\end{tabular}
\end{subtable}
\hspace*{15pt}
\begin{subtable}{0.45\textwidth}
\begin{tabular}{cccc}
\toprule
 & \gls{RS-VI} $B=1$ & \gls{RS-VI} $B=4$ &\gls{G-REP}\\ \hline
Min & $1.8\mathrm{e}{-3}$ & $1.5\mathrm{e}{-3}$ & $2.6\mathrm{e}{-3}$\\
Median & $1.2\mathrm{e}{4}$ & $\mathbf{4.5\mathrm{e}{3}}$ & $1.5\mathrm{e}{7}$\\
Max & $1.4\mathrm{e}{12}$ & $\mathbf{1.6\mathrm{e}{11}}$ & $3.5\mathrm{e}{12}$ \\ \bottomrule
\end{tabular}
\end{subtable}
\caption{The \gls{RS-VI} gradient (this paper) exhibits lower variance than \gls{G-REP} \citep{RuizTB2016}. We show estimated variance, based on $10$ samples, of \gls{G-REP} and \gls{RS-VI} (for $B=1,4$ shape augmentation steps), for parameters at the initialization point (\emph{left}) and at iteration $2600$ in \gls{RS-VI} (\emph{right}), estimated for the \gls{NIPS} data.}\label{tab:vardef}
\end{table*}

\parhead{Model.}
The sparse gamma \gls{DEF} \citep{Ranganath2015} is a multi-layered probabilistic model that mimics the architecture of deep neural networks. It models the data using a set of local latent variables $z_{n,k}^\ell$ where $n$ indexes observations, $k$ components, and $\ell$ layers. These local variables are connected between layers through global weights $w_{k,k'}^\ell$. The observations are $x_{n,d}$, where $d$ denotes dimension. The joint probabilistic model is defined as
\begin{align}
\begin{split}
z_{n,k}^\ell &\sim \gam\left(\alpha_z,\frac{\alpha_z}{\sum_{k'} w_{k,k'}^\ell z_{n,k'}^{\ell+1}} \right), \\
x_{n,d} &\sim \mathrm{Poisson}\left(\sum_{k} w_{k,d}^0 z_{n,k}^1 \right).
\end{split}
\end{align}
We set $\alpha_z=0.1$ in the experiments. All priors on the weights are set to $\gam(0.1,0.3)$, and the top-layer local variables priors are set to $\gam(0.1,0.1)$. We use $3$ layers, with $100$, $40$, and $15$ components in each. 
This is the same model that was studied by \citet{RuizTB2016}, where \gls{G-REP} was shown to outperform both \gls{BBVI} (with control variates and Rao-Blackwellization), as well as \gls{ADVI}. In the experiments we follow their approach and parameterize the variational approximating gamma distribution using the shape and mean. To avoid constrained optimization we use the transform $\theta = \log(1+\exp(\vartheta))$ for non-negative variational parameters $\theta$, and optimize $\vartheta$ in the unconstrained space.

\parhead{Results.}
For the Olivetti faces we explore ${\eta\in\{0.75,1,2,5\}}$ and show the resulting \gls{ELBO} of the best one in Figure~\ref{fig:olivetteDEF}. We can see that \gls{RS-VI} has a significantly faster initial improvement than any of the other methods.\footnote{%
	The results of \gls{G-REP}, \gls{ADVI} and \gls{BBVI} where reproduced with permission from \citet{RuizTB2016}.%
} 
The wall-clock time for \gls{RS-VI} is based on a Python implementation (average $1.5$s per iteration) using the automatic differentiation package autograd \citep{autograd}. We found that \gls{RS-VI} is approximately two times faster than \gls{G-REP} for comparable implementations. One reason for this is that the transformations based on rejection sampling are cheaper to evaluate. Indeed, the research literature on rejection sampling is heavily focused on finding cheap and efficient transformations.


For the \gls{NIPS} dataset, we now compare the variance of the gradients between the two estimators, \gls{RS-VI} and \gls{G-REP}, for different shape augmentation steps $B$. In Table~\ref{tab:vardef} we show the minimum, median, and maximum values of the variance across all dimensions. We can see that \gls{RS-VI} again clearly outperforms \gls{G-REP} in terms of variance. Moreover, increasing the number of augmentation steps $B$ provides even further improvements.




\section{Conclusions}
\glsresetall

We introduced \gls{RS-VI}, a method for deriving reparameterization
gradients when simulation from the variational distribution is done
using a acceptance-rejection sampler.  In practice, \gls{RS-VI} leads to
lower-variance gradients than other state-of-the-art methods.
Further, it enables reparameterization gradients for a large class of
variational distributions, taking advantage of the efficient
transformations developed in the rejection sampling literature.

This work opens the door to other strategies that ``remove the lid''
from existing black-box samplers in the service of variational
inference.  As future work, we can consider more complicated
simulation algorithms with accept-reject-like steps, such as adaptive
rejection sampling, importance sampling, sequential Monte Carlo, or
Markov chain Monte Carlo.





\subsubsection*{Acknowledgements}
Christian A.\ Naesseth is supported by CADICS, a Linnaeus Center, funded by the Swedish Research Council (VR). Francisco J.\ R.\ Ruiz is supported by the EU H2020 programme (Marie Sk\l{}odowska-Curie grant agreement 706760). Scott W. Linderman is supported by the Simons Foundation SCGB-418011.
This work is supported by NSF IIS-1247664, ONR N00014-11-1-0651, DARPA
PPAML FA8750-14-2-0009, DARPA SIMPLEX N66001-15-C-4032, Adobe, and the
Alfred P. Sloan Foundation. The authors would like to thank Alp Kucukelbir and Dustin Tran for helpful comments and discussion.


\bibliographystyle{abbrvnat}
\bibliography{refs}

\newpage
\appendix
\section{Supplementary Material}
\subsection{Distribution of $\eps$}
Here we formalize the claim in the main manuscript regarding the distribution of the accepted variable $\eps$ in the rejection sampler. Recall that ${z=h(\eps,\theta), ~\eps \sim s(\eps)}$ is equivalent to $z\sim r(z\g\theta)$, and that $q(z\g\theta) \leq M_\theta r(z\g\theta)$. For simplicity we consider the univariate continuous case in the exposition below, but the result also holds for the discrete and multivariate settings. The cumulative distribution function for the accepted $\eps$ is given by
\begin{align*}
\begin{split}
&\Prb(E \leq \eps) = \sum_{i=1}^\infty \Prb(E \leq \eps, E = E_i)\\
&= \sum_{i=1}^\infty \Bigg[ \Prb\left(E_i \leq \eps, U_i < \frac{q(h(E_i,\theta)\g\theta)}{M_\theta r(h(E_i,\theta)\g\theta)}\right)\\
&\quad\prod_{j=1}^{i-1} \Prb\left(U_j \geq \frac{q(h(E_j,\theta)\g\theta)}{M_\theta r(h(E_j,\theta)\g\theta)}\right)\Bigg] \\
&= \sum_{i=1}^\infty \int_{-\infty}^\eps s(e) \frac{q(h(e,\theta)\g\theta)}{M_\theta r(h(e,\theta)\g\theta)} \myd e \prod_{j=1}^{i-1} \left(1-\frac{1}{M_\theta}\right)\\
&= \int_{-\infty}^\eps s(e) \frac{q(h(e,\theta)\g\theta)}{r(h(e,\theta)\g\theta)} \myd e \cdot \frac{1}{M_\theta} \cdot \sum_{i=1}^\infty  \left(1-\frac{1}{M_\theta}\right)^{i-1}\\
&= \int_{-\infty}^\eps s(e) \frac{q(h(e,\theta)\g\theta)}{r(h(e,\theta)\g\theta)}\myd e.
\end{split}
\end{align*}
Here, we have applied that ${z=h(\eps,\theta), ~\eps \sim s(\eps)}$ is a reparameterization of $z\sim r(z\g\theta)$, and thus
\begin{align*}
\begin{split}
& \Prb\left(U_j \geq \frac{q(h(E_j,\theta)\g\theta)}{M_\theta r(h(E_j,\theta)\g\theta)}\right) \\
& = \int_{-\infty}^{\infty} s(e)\left( 1- \frac{q(h(e,\theta)\g\theta)}{M_\theta r(h(e,\theta)\g\theta)}\right) \myd e \\
& = 1-\frac{1}{M_\theta}\E_{s(e)}\left[ \frac{q(h(e,\theta)\g\theta)}{r(h(e,\theta)\g\theta)}\right] \\
& = 1-\frac{1}{M_\theta}\E_{r(z\g \theta)} \left[ \frac{q(z\g\theta)}{r(z\g\theta)} \right] = 1-\frac{1}{M_\theta}.
\end{split}
\end{align*}

The density is obtained by taking the derivative of the cumulative distribution function with respect to $\eps$,
\begin{align*}
\frac{\myd }{\myd \eps} \Prb(E \leq \eps) =  s(\eps) \frac{q(h(\eps,\theta)\g\theta)}{r(h(\eps,\theta)\g\theta)},
\end{align*}
which is the expression from the main manuscript.

The motivation from the main manuscript is basically a standard ``area-under-the-curve'' or geometric argument for rejection sampling \citep{robert2004monte}, but for $\eps$ instead of $z$.

\subsection{Derivation of the Gradient}\label{sec:gradient}
We provide below details for the derivation of the gradient. We assume that $h$ is differentiable (almost everywhere) with respect to $\theta$, and that ${f(h(\eps,\theta))\frac{q(h(\eps,\theta)\g\theta)}{r(h(\eps,\theta)\g\theta)}}$ is continuous in $\theta$ for all $\eps$. Then, we have
\begin{align*}
\begin{split}
&\grad_\theta \E_{q(z\g \theta)}[f(z)] = \grad_\theta \E_{\pi(\eps\g\theta)}[f(h(\eps,\theta))] \\
&= \int s(\eps) \grad_\theta \left( f\left(h(\eps,\theta)\right)  \frac{q\left(h(\eps,\theta)\g \theta\right)}{r\left(h(\eps,\theta)\g \theta\right)}  \right) \myd \eps\\
&= \int s(\eps)  \frac{q\left(h(\eps,\theta)\g \theta\right)}{r\left(h(\eps,\theta)\g \theta\right)}  \grad_\theta f\left(h(\eps,\theta)\right)  \myd \eps\\
&+ \int s(\eps) f\left(h(\eps,\theta)\right)  \grad_\theta \left( \frac{q\left(h(\eps,\theta)\g \theta\right)}{r\left(h(\eps,\theta)\g \theta\right)} \right) \myd \eps\\
&= \underbrace{\E_{\pi(\eps\g\theta)}[\grad_\theta f(h(\eps,\theta))]}_{\defeq g_{\text{rep}}}+ \\
&\quad + \underbrace{\E_{\pi(\eps\g\theta)}\left[ f(h(\eps,\theta)) \grad_\theta  \log \frac{q(h(\eps,\theta)\g \theta)}{r(h(\eps,\theta)\g \theta)} \right]}_{\defeq g_{\text{cor}}},
\end{split}
\end{align*}
where in the last step we have identified $\pi(\eps\g\theta)$ and made use of the log-derivative trick
\begin{align*}
 \grad_\theta \frac{q\left(h(\eps,\theta)\g \theta\right)}{r\left(h(\eps,\theta)\g \theta\right)} = \frac{q\left(h(\eps,\theta)\g \theta\right)}{r\left(h(\eps,\theta)\g \theta\right)} \grad_\theta \log \frac{q\left(h(\eps,\theta)\g \theta\right)}{r\left(h(\eps,\theta)\g \theta\right)}.
\end{align*}

\paragraph{Gradient of Log-Ratio in $g_{\text{cor}}$}
For invertible reparameterizations we can simplify the evaluation of the gradient of the log-ratio in $g_{\text{cor}}$ as follows using standard results on transformation of a random variable
\begin{align*}
&\grad_\theta  \log \frac{q(h(\eps,\theta)\g \theta)}{r(h(\eps,\theta)\g \theta)}  = \grad_\theta  \log q(h(\eps,\theta)\g \theta) +\nonumber\\
& +\grad_\theta \log \left| \frac{d h}{d\eps}(\eps,\theta) \right|- \grad_\theta \log \underbrace{s(h^{-1}(h(\eps,\theta),\theta))}_{= ~s(\eps)}\\
&=\grad_\theta  \log q(h(\eps,\theta)\g \theta) + \grad_\theta \log \left| \frac{d h}{d\eps}(\eps,\theta) \right|.
\end{align*}

\renewcommand{\arraystretch}{2.2}
\begin{table*}[t]
\centering
\begin{tabular}{ccc} 
\toprule
$q(z\g\theta)$ & $h(\eps,\theta)$ & $s(\eps)$\\
\hline
$\gam(\alpha,1)$ &$\left(\alpha-\frac{1}{3}\right) \left(1+\frac{\eps}{\sqrt{9\alpha-3}}\right)^3$ & $\eps\sim\N(0,1)$\\
$\tN(0,1,a,\infty)$ & $\sqrt{a^2-2\log \eps}$ & $\eps \sim \uni[0,1]$\\
$\mises(\kappa)$ & $\sign(\eps)\arccos\left( \frac{1+c\cos(\pi \eps)}{c+\cos(\pi\eps)} \right)$, $c = \frac{1+\rho^2}{2\rho}$, $\rho = \frac{r-\sqrt{2r}}{2\kappa}$, $r = 1+\sqrt{1+4\kappa^2}$ & $\eps \sim \uni[-1,1]$ \\ \bottomrule
\end{tabular}
\caption{Examples of reparameterizable rejection samplers; many more can be found in \citet{devroye1986}. The first column is the distribution, the second column is the transformation $h(\eps,\theta)$, and the last column is the proposal $s(\eps)$.\label{tab:rep_rejsamplers}}
\end{table*}
\begin{table*}[t]
\centering
\begin{tabular}{ccc}
\toprule
$q(z\g\theta)$ & $g(\tilde z, \theta)$ &  $p(\tilde z\g\theta)$\\
\hline
$\bet(\alpha,\beta)$ & $\displaystyle\frac{\tilde z_1}{\tilde z_1 + \tilde z_2}$ & $\tilde z_1 \sim \gam(\alpha,1)$, $\tilde z_2  \sim \gam(\beta,1)$\\
$\diri(\alpha_{1:K})$ &$\displaystyle\frac{1}{\sum_\ell \tilde z_\ell} \left(\tilde z_1,\ldots,\tilde z_K\right)^\top$ & $\tilde z_k \sim \gam(\alpha_k,1), ~k=1,\ldots,K$\\
$\St(\nu)$ & $\displaystyle\sqrt{\frac{\nu}{2 \tilde z_1}} \tilde z_2$ & $\tilde z_1 \sim \gam(\nu/2,1)$, $\tilde z_2 \sim \N(0,1)$\\
$\chi^2(k)$ & $\displaystyle2\tilde z$& $\tilde z \sim \gam(k/2,1)$\\
$\fdist(d_1,d_2)$ & $\displaystyle\frac{d_2 \tilde z_1}{d_1 \tilde z_2}$ & $\tilde z_1 \sim \gam(d_1/2,1)$, $\tilde z_2  \sim \gam(d_2/2,1)$\\
$\nakagami(m,\Omega)$ & $\displaystyle \sqrt{\frac{\Omega\tilde z}{m}}$ & $\tilde z \sim \gam(m,1)$ \\ \bottomrule
\end{tabular}
\caption{Examples of random variables as functions of auxiliary random variables with reparameterizable distributions. The first column is the distribution, the second column is a function $g(\tilde{z},\theta)$ mapping from the auxiliary variables to the desired variable, and the last column is the distribution of the auxiliary variables $\tilde z$.\label{tab:aux_rejSamplers}}
\end{table*}

\subsection{Examples of Reparameterizable Rejection Samplers}

We show in Table~\ref{tab:rep_rejsamplers} some examples of reparameterizable rejection samplers for three distributions, namely, the gamma, the truncated normal, and the von Misses distributions (for more examples, see \citet{devroye1986}). We show the distribution $q(z\g \theta)$, the transformation $h(\eps,\theta)$, and the proposal $s(\eps)$ used in the rejection sampler.

We show in Table~\ref{tab:aux_rejSamplers} six examples of distributions that can be reparameterized in terms of auxiliary gamma-distributed random variables. We show the distribution $q(z\g\theta)$, the distribution of the auxiliary gamma random variables $p(\tilde{z}\g\theta)$, and the mapping $z=g(\tilde{z},\theta)$.

\subsection{Reparameterizing the Gamma Distribution}
We provide details on reparameterization of the gamma distribution. In the following we consider rate $\beta=1$. Note that this is not a restriction, we can always reparameterize the rate. The density of the gamma random variable is given by
\begin{align*}
q(z\g\alpha) &= \frac{z^{\alpha-1} e^{-z}}{\Gamma(\alpha)},
\end{align*}
where $\Gamma(\alpha)$ is the gamma function. We make use of the reparameterization defined by
\begin{align*}
z &= h(\eps,\alpha) = \left(\alpha-\frac{1}{3}\right)\left(1+\frac{\eps}{\sqrt{9\alpha-3}}\right)^3,\\
\eps &\sim \N(0,1).
\end{align*}
Because $h$ is invertible we can make use of the simplified gradient of the log-ratio derived in Section~\ref{sec:gradient} above. The gradients of $\log q$ and $-\log r$ are given by
\begin{align*}
&\grad_\alpha  \log q(h(\eps,\alpha)\g \alpha) \\
&= \log (h(\eps,\alpha)) + (\alpha-1) \frac{\frac{d h(\eps,\alpha)}{d\alpha}}{h(\eps,\alpha)} -\frac{d h(\eps,\alpha)}{d\alpha} -\psi(\alpha),\\
&\grad_\alpha - \log r(h(\eps,\alpha)\g\alpha) = \grad_\alpha \log \left| \frac{d h}{d\eps}(\eps,\alpha) \right| \\
&= \frac{1}{2\left(\alpha-\frac{1}{3}\right)} - \frac{9 \eps}{\left(1+\frac{\eps}{\sqrt{9\alpha-3}}\right)\left(9\alpha-3\right)^{\frac{3}{2}}},
\end{align*}
where $\psi(\alpha)$ is the digamma function and
\begin{align*}
&\frac{d h(\eps,\alpha)}{d\alpha} \\
&= \left(1+\frac{\eps}{\sqrt{9\alpha-3}}\right)^3 - \frac{27 \eps}{2(9\alpha-3)^{\frac{3}{2}}}\left(1+\frac{\eps}{\sqrt{9\alpha-3}}\right)^2.
\end{align*}

\end{document}